\documentclass{article}
\usepackage{arxiv}
\AtBeginDocument{%
  \providecommand\BibTeX{{%
    \normalfont B\kern-0.5em{\scshape i\kern-0.25em b}\kern-0.8em\TeX}}}

\usepackage{amsmath,amssymb,amsthm}
\usepackage{url}
\usepackage{datetime2}
\usepackage{graphicx} 
\usepackage{amsmath,dsfont,amsthm}
\usepackage[vlined,linesnumbered,ruled,titlenotnumbered]
{algorithm2e}
\newcommand{\gsemo}{GSEMO\xspace}
\newcommand{\bcgsemo}{BC-GSEMO\xspace}

\newcommand{\temax}{t_{\mathrm{epoch}}}
\usepackage{color}

\newcommand{\ignore}[1]{}

\newtheorem{theorem}{Theorem}
\newtheorem{lemma}{Lemma}

\newcommand{\oea}{\mbox{${(1 + 1)}$~EA}\xspace}

\newcommand{\LO}{\textsc{Leading\-Ones}\xspace}
\newcommand{\leadingones}{\LO}

\newcommand{\N}{\ensuremath{\mathbb{N}}} % ohne Null!!!

\DeclareMathOperator{\Geom}{Geom}

\let\originalleft\left
\let\originalright\right
\renewcommand{\left}{\mathopen{}\mathclose\bgroup\originalleft}
\renewcommand{\right}{\aftergroup\egroup\originalright}

\title{A Block-Coordinate Descent EMO Algorithm: Theoretical and Empirical Analysis}

\author{
Benjamin Doerr\\
Laboratoire d’Informatique (LIX),\\
    CNRS, 
    École Polytechnique\\ 
    Institut Polytechnique de Paris\\
    Palaiseau, France\And
Joshua Knowles\\ 
Schlumberger Cambridge Research\\SLB\\
    Cambridgeshire, United Kingdom
\And
  Aneta Neumann\\
  Optimisation and Logistics,\\
 School of Computer and Mathematical Sciences,\\
  The University of Adelaide, Australia\\
  \And
  Frank Neumann\\
  Optimisation and Logistics,\\
 School of Computer and Mathematical Sciences,\\
  The University of Adelaide, Australia\\
}

\begin{document}

\maketitle
\begin{abstract}
We consider whether conditions exist under which block-coordinate descent is asymptotically efficient in evolutionary multi-objective optimization, addressing an open problem. Block-coordinate descent, where an optimization problem is decomposed into $k$ blocks of decision variables and each of the blocks is optimized (with the others fixed) in a sequence, is a technique used in some large-scale optimization problems such as airline scheduling, however its use in multi-objective optimization is less studied. We propose a block-coordinate version of GSEMO and compare its running time to the standard GSEMO algorithm. Theoretical and empirical results on a bi-objective test function, a variant of LOTZ, serve to demonstrate the existence of cases where block-coordinate descent is faster. The result may yield wider insights into this class of algorithms.
\end{abstract}

\keywords{block coordinate decent, evolutionary multi-objective optimization, runtime analysis, theory}

\maketitle

\section{Introduction}

Evolutionary algorithms have been widely used to tackle multi-objective optimization problems~\cite{coello2000updated,corne2003good,DBLP:reference/sp/Deb15}. This includes the use of popular algorithms such as NSGA-II~\cite{DBLP:journals/tec/DebAPM02}, SMS-EMOA~\cite{BeumeNE07} and SPEA2~\cite{zitzler2001spea2} for problems with two or three objectives, and approaches such as NSGA-III~\cite{DebJain2014} and MOEA/D~\cite{QingfuZhang2007} for many-objective problems. 
In terms of theoretical understanding, the area of rigorous runtime analysis~\cite{NeumannW10,AugerD11,Jansen13,ZhouYQ19,DoerrN20} has provided a sound understanding of different algorithmic components~\cite{DBLP:series/sci/HorobaN10} and problem formulations~\cite{FriedrichHHNW10,DBLP:journals/ai/RoostapourNNF22}, and more recently the overall performance of popular algorithms such as NSGA-II~\cite{ZhengD23aij,CerfDHKW23}, NSGA-III~\cite{WiethegerD23,NSGA-IIImulti}, MOEA/D~\cite{HuangZCH19,Huang2021,do2023rigorous}, and SMS-EMOA~\cite{BianZLQ23,ZhengD24}.

In this paper, we explore the use of a technique called \emph{block coordinate descent} for tackling large-scale problems in the context of evolutionary multi-objective optimization.
Block-coordinate descent is a technique that has been shown to be particularly useful in the context of classical optimization when the problem can be structured according to different subproblems. While this technique allows for significant speed-ups and better overall results in many situations, it has not been explored until quite recently in the context of classical multi-objective programming~\cite{ DagstuhlWG,beck2013convergence,schobel2017eigenmodel,jager2020blockwise}   or evolutionary multi-objective optimization~\cite{DagstuhlWG}.

We provide a set-up and demonstrate, through a class of example functions, where the block coordinate evolutionary multi-objective approach has a significant benefit over the standard approach.
This benchmark function is a block-wise variant that incorporates the classical \LO problem and building on the single-objective benchmark function studied in \cite{DoerrSW13foga,DBLP:conf/gecco/0001KN0R23}.
Our multi-objective problem consists of $k$ blocks and has a Pareto front of size $2^k$. We examine the time until all such trade-offs have been obtained.
The problem is structured in such a way that blocks with an index $i$ contribute more to the two objective functions than the sum of all blocks with indices $j>i$.
We show that the block-coordinate approach allows for parallel optimization of the different blocks. In contrast, a standard approach tends to destroy the beneficial structure in lower priority blocks during the run when achieving improvements in higher priority blocks.

In order to show this, we analyze the classical \gsemo algorithm and the block-coordinate variant \bcgsemo introduced here on the benchmark functions described.
We use rigorous runtime analysis and provide a lower bound for \gsemo and an upper bound for \bcgsemo, which shows that \bcgsemo is faster in optimizing the problem. The reason for this behavior is that \gsemo, which tackles the whole problem at once, tends to destroy components in different blocks that might later be optimized when the population size of \gsemo has increased significantly. Compared to this, \bcgsemo does not encounter this problem and optimizes the different blocks in parallel while not destroying the solution structure of other blocks.

We also confirm from an experimental perspective the usefulness of the block-coordinate approach in the context of evolutionary multi-objective optimization.
In order to show that the theoretically observed behavior already occurs for relatively small problem sizes, we carry out experimental investigations for a small number of up to $k=4$ blocks. The experiments show a significantly different scaling behavior of \gsemo and \bcgsemo for all settings investigated.

The paper is structured as follows. In Section~\ref{sec2}, we present the basics and introduce our block-coordinate descent approach incorporated into the \gsemo algorithm. Section~\ref{sec3} introduces our class of benchmark problems and points out important structural properties. We carry out the runtime analysis of \gsemo and \bcgsemo for the problem introduced in Section~\ref{sec4} and complement our theoretical findings with an experimental study in Section~\ref{sec5}. Finally, we finish with some discussion and concluding remarks.

\section{Preliminaries and Block coordinate \gsemo}
\label{sec2}
We consider the classical bi-objective task of maximizing two conflicting objective functions given the search space $S=\{0,1\}^n$.
The fitness of a search point $x \in S$ is given as $f=(f_1, f_2) \colon S \rightarrow \mathds{R}^2$.
A solution $x$ (weakly) dominates  a solution $y$ (denoted as $x \succeq y$) iff $f_i(x) \geq f_i(x)$, $1 \leq i \leq 2$. A solution $x$ strongly dominates a solution $y$ (denoted as $x \succ y$) iff $x \succeq y$ and $f_i(x) > f_i(x)$ for at least one $i \in \{1,2\}$ holds. 
A search points $x$ is called Pareto optimal if there is no other search point $y$ with $y \succ x$. 
The notion of dominance translates from the search points to their objective vectors.
The set of all Pareto optimal search points $X^* \subseteq 2^S$ is called the Pareto set and the set of corresponding objective vectors is called the Pareto front~$F$.
The classical goal in multi-objective optimization is to compute for each objective vector of $q\in F$ a solution $x \in X^*$ with $f(x)=q$.

The \gsemo algorithm (see Algorithm~\ref{alg:GSEMO}) is a simple evolutionary multi-objective algorithm that has been frequently studied in the area of runtime analysis~\cite{DBLP:conf/cec/Giel03,giel2006effect,DBLP:journals/tec/BrockhoffFHKNZ09,FriedrichHHNW10,DBLP:conf/cec/DoerrKV13,DoerrZ21aaai,DinotDHW23}. It starts with a population consisting of a single solution chosen uniformly at random. In each iteration, a solution $x$ is chosen uniformly at random from the population $P$, and an offspring $y$ is produced from $x$ by mutation. Mutation is usually standard bit mutation which flips each bit independently with probability $1/n$. The offspring $y$ is accepted and added to the population if it is not strongly dominated by other search points in $P$. If $y$ is added, then all other search points in $P$ that are weakly dominated by $y$ are removed from $P$.

\subsection*{\bcgsemo}
Throughout the paper we consider a problem where we decompose the search points $x=(x_{B_1}, \ldots, x_{B_k})$ into $k$ blocks $x_{B_1}, \ldots, x_{B_k}$. We assume that each block consists of $\ell$ bits and therefore have $n=k \cdot \ell$. 
The \bcgsemo algorithm is similar to \gsemo. However, it applies mutation only to a chosen block. 
%When we apply block coordinate mutation to block $i$, we flip each bit of $x_{B_i}$ independently of the others with probability $1/\ell$.
As each block has length $\ell$, mutation flips each bit in the chosen block with probability $1/\ell$.
Blocks are mutated consecutively by applying the mutation operator to any given block $\temax$ times before moving on to the next one (see inner repeat-loop). 
The parameter $\temax$ is assuming to be a value significantly smaller than the overall runtime which implies that each block is mutated roughly equally often. It should be noted that this scheme of applying mutation to each block a $k$-th fraction of the time and if so flipping each bit with probability $1/\ell$ %(instead of $1/n$ in \gsemo) 
matches the mutation probability of $1/n$ for \gsemo in terms of how often a particular bit is flipped.

When analyzing \gsemo and \bcgsemo with respect to their runtime behavior, we measure time in terms of the expected number of fitness evaluations to reach a given goal. The expected optimization time refers to the expected number of fitness evaluations until the particular algorithm has obtained an optimal population, i.e. a population containing for each Pareto optimal objective vector a corresponding solution, for the first time.

%%%%%%%%%%%%%%%%%%%
\begin{algorithm}[!t]
Initialize $x \in \{0,1\}^n$ uniformly at random\;
 $P\leftarrow \{x\}$\;
\Repeat{$\mathit{stop}$}{
Choose $x\in P$ uniformly at random\;
Create $y$ from $x$ by mutation\;
\If{$\nexists\, w \in P: w \succ y$} {
  $P \leftarrow (P \setminus \{z\in P \mid y \succeq z\}) \cup \{y\}$\;
  } }
\caption{Global simple evolutionary multi-objective optimizer (GSEMO)} 
\label{alg:GSEMO}
\end{algorithm}

\begin{algorithm}[t]
Initialize $x \in \{0,1\}^n$ uniformly at random\;
 $P\leftarrow \{x\}$\;
 
$i \leftarrow 0$\;
 %$\hat{P} \leftarrow P$\;
 \Repeat{stop}{
 $t \leftarrow 0$\;
\Repeat{$\mathit{t\geq \temax}$}{
$t \leftarrow t+1$\;
Choose $x \in P$ uniformly at random\;
Create $y$ from $x$ by block-coordinate mutation on block $x_{i+1}$\;
\If{$\nexists\, w \in P: w \succ y$} {
  $P \leftarrow (P \setminus \{z\in P \mid y \succeq z\}) \cup \{y\}$\;
  } }
  $i \leftarrow (i+1) \mod k$\;
  }
\caption{Block coordinate GSEMO (BC-GSEMO)} \label{alg:GSEMO-bc}
\end{algorithm}

\section{Benchmark Function and Properties}
\label{sec3}\label{sec:benchmark}

We now define the benchmark problem we shall regard in this work, and for which we shall show that a block-coordinate approach is superior to a classic multi-objective approach.

Let $n \in \N$ and $k$ be a divisor of~$n$. Let $\ell := n/k$. For a bitstring $x \in \{0,1\}^n$, we write $x = (x_{B_1}, \ldots, x_{B_k})$, where $x_{B_i} = (x_{(i-1)\ell+1}, \dots, x_{i\ell})$ is the $i$th block of $x$.

For all $z \in \{0,1\}^\ell$, we define 
\[
LO_z(x_{B_i}) =\sum_{i=1}^{\ell} \prod_{j=1}^i (z_i=x_i)
\]
as the number of leading positions in which $x_{B_i}$ agrees with $z$. We note that this function is isomorphic to the classic LeadingOnes problem, first proposed in~\cite{Rudolph97} and then intensively studied in the theory of evolutionary algorithms, see, e.g., \cite{Doerr19tcs}. We shall be interested in two particular target strings. Let $r \le \ell$. Let $z^1 = 1^{\ell}$ and $z^2 = 1^{\ell-r}0^r$.

We are now ready to define the bi-objective problem $f_{nkr}$ on the search space $\{0,1\}^n$. To ease the notation, we suppress the indices $n$, $k$, and $r$ and simply write $f$ whenever there is no risk for confusion. We define $f =(f_1, f_2) \colon \{0,1\}^n \rightarrow \mathds{R}^2$ via
\begin{align*}
f_1(x) &= \sum_{b=1}^k \left((\ell+1)^{2(k-b)+1}  LO_{z^1}(x_{B_b}) + (\ell+1)^{2(k-b)}  LO_{z^2}(x_{B_b})\right),\\
f_2(x) &= \sum_{b=1}^k \left((\ell+1)^{2(k-b)+1}  LO_{z^2}(x_{B_b}) + (\ell+1)^{2(k-b)}  LO_{z^1}(x_{B_b})\right).
\end{align*}
To analyze these functions, let us define the \emph{base functions}
\begin{align*}
  f^B_1(x_B) &=  (\ell+1)  LO_{z^1}(x_{B}) + LO_{z^2}(x_{B}),\\
  f^B_2(x_B) &=  (\ell+1)  LO_{z^2}(x_{B}) + LO_{z^1}(x_{B}),
\end{align*}
for all $x_B \in \{0,1\}^\ell$.
Then for both objectives $m \in [1..2]$, we can write 
\begin{equation}\label{eq:fsumbasic}
f_m(x) = \sum_{b=1}^k (\ell+1)^{2(k-b)} f_m^b(x_{B_b}).    %https://www.overleaf.com/project/64f9c1db5e8950560121aef2
\end{equation}

We collect some elementary properties of $f$ and its ingredients. Right from the definition, we obtain the following observation on how different bitstrings compare under the $f_m$, $m = 1, 2$. 
\begin{lemma}\label{lem:fbasic}
  Let $x_B, y_B \in \{0,1\}^\ell$ with $x_B \neq y_B$. Let $i \in [1..\ell]$ be the first bit in which $x_B$ and $y_B$ differ, and let $(x_B)_i = 1$ (and thus $(y_B)_i = 0$). Then the following is true.
  \begin{enumerate}
    \item If $i < \ell-r+1$, then $f_1^B(x_B) > f_1^B(y_B)$ and $f_2^B(x_B) > f_2^B(y_B)$.  
    \item If $i = \ell-r+1$, then $f_1^B(x_B) > f_1^B(y_B)$ and $f_2^B(x_B) < f_2^B(y_B)$.  
    \item If $i > \ell-r+1$, then one of the following three cases applies.
    \begin{enumerate}
      \item If $(x_B)_{\ell-r+1} = \dots = (x_B)_{i-1} = 1$ (and similarly for $y_B$), then $f_1^B(x_B) > f_1^B(y_B)$ and $f_2^B(x_B) > f_2^B(y_B)$. 
    \item If $(x_B)_{\ell-r+1} = \dots = (x_B)_{i-1} = 0$ (and similarly for $y_B$), then $f_1^B(x_B) < f_1^B(y_B)$ and $f_2^B(x_B) < f_2^B(y_B)$.
    \item Otherwise, $f_1^B(x_B) = f_1^B(y_B)$ and $f_2^B(x_B) = f_2^B(y_B)$.
    \end{enumerate} 
  \end{enumerate}
  In particular, if $x_B$ and $y_B$ have equal value under one objective function, then also under the other.  
\end{lemma}

The exponentially decreasing coefficients of the $f_m^b$ in~\eqref{eq:fsumbasic} allow to extend, in a suitable way, the previous lemma to the full function $f_m$.

\begin{lemma}\label{lem:f}
  Let $m \in \{1,2\}$ and $\overline m = 3-m$ be the other objective. Let $x, y \in \{0,1\}^n$ with $f_m(x) \neq f_m(y)$. Let $b$ be minimal such that $f_m^B(x_{B_b}) \neq f_m^B(y_{B_b})$. Then
  \begin{enumerate}
    \item $f_{\overline m}^B(x_{B_b}) \neq f_{\overline m}^B(y_{B_b})$;
    \item $b$ is minimal subject to $f_{\overline m}^B(x_{B_b}) \neq f_{\overline m}^B(y_{B_b})$;
    \item the function values of $x$ and $y$ compare like the base function values of $x_{B_b}$ and $y_{B_b}$, that is, for all $m' \in \{1,2\}$ we have 
    \begin{itemize}
     \item $f_{m'}(x) < f_{m'}(y)$ if and only if $f_{m'}^B(x_{B_b}) < f_{m'}^B(y_{B_b})$, and 
     \item $f_{m'}(x) > f_{m'}(y)$ if and only if $f_{m'}^B(x_{B_b}) > f_{m'}^B(y_{B_b})$.
    \end{itemize}
  \end{enumerate}
\end{lemma}

\begin{proof}
  The first two claims follow immediately from the last clause of Lemma~\ref{lem:fbasic}. For the third claim, note that for any $z \in \{0,1\}^n$ the contribution of the all blocks $B_c$, $c > b$, to $f_{m'}$ is at most 
  \begin{align*}
  \sum_{c = b+1}^{k} (\ell+1)^{2(k-c)} f_{m'}^B(z_{B_c}) 
  & \le \sum_{c = b+1}^{k} (\ell+1)^{2(k-c)}((\ell+1) \ell + \ell) \\
  & = \ell \sum_{i=0}^{2(k-b-1)+1} (\ell+1)^i < (\ell+1)^{2(k-b)}.
  \end{align*}
  In contrast, any difference between $f_{m'}^B(x_{B_b})$ and $f_{m'}^B(y_{B_b})$ immediately leads to a contribution to $f_{m'}$ that differs by at least $2^{2(k-b)}$, thus outnumbering any possibly different contributions of the later blocks.  By assumption and the first two claims, the earlier blocks contribute the same to $f_{m'}$. This proves the third claim.
\end{proof}

\begin{lemma}\label{lem:popsize}
  Let $P \subseteq \{0,1\}^n$ be a set of pair-wise non-dominating solutions. Then $|P| \le 2^k$.
\end{lemma}

\begin{proof}
  By induction, we prove, for all $i \in [0..k]$, the following assertion.
  \begin{description}
  \item[$(A_i)$] Let $P \subseteq \{0,1\}^n$ be any set of pair-wise non-dominating solution such that for all $b \in [1..k-i]$, we have that $f_1^B(x_{B_b}) = f_1^B(y_{B_b})$ for all $x, y \in P$. Then $|P| \le 2^i$.
  \end{description}
  We note that $A_k$ is exactly the claim of the lemma, hence proving $(A_i)$ for all $i \in [0..k]$ also proves the lemma.

  Clearly, we have $(A_0)$. In this case, the assumptions of $(A_i)$ imply that all members of $P$ have the same $f_1$ value. Since $P$ is pair-wise non-dominating, this implies $|P| = 1 = 2^i$. 

  Hence assume that for some $i \in [1..k]$ we have shown already $(A_{i-1})$. We use this to show $(A_i)$. We regard the set $V = \{f_1^B(x_{B_{k-i}}) \mid x \in P\}$ of different objective values in the $(k-i)$-th block. If $|V| = 1$, then $P$ satisfies the assumptions of $(A_{i-1})$ and we immediately conclude $|P| \le 2^{i-1} \le 2^i$. Similarly, if $|V| = 2$, say $V = \{v_1, v_2\}$, then the sets $P_j = \{x \in P \mid f_1^B(x_{B_{k-i}}) = v_j\}$, $j = 1,2$, satisfy the assumptions of $(A_{i-1})$, we conclude $|P_j| \le 2^{i-1}$, and thus $|P| = |P_1| + |P_2| \le 2^i$. 
  
  We conclude the proof by showing that $|V| \ge 3$ contradicts our assumption that $P$ is pair-wise non-dominating. Hence assume $|V| \ge 3$. Assume first that there is an $x \in P$ such that $LO(x_{B_{k-i}}) < \ell-k$. Then all other $y \in P$ satisfy $LO(y_{B_{k-i}}) = LO(x_{B_{k-i}})$ as otherwise the first case of Lemma~\ref{lem:fbasic} would apply and show a contradiction to our assumption that $P$ is pairwise non-dominating. However, if all $y \in P$ have the same value $LO(y_{B_{k-i}})$ that is less than $\ell-k$, then by definition of $f_1^B$ all these $y$ have the same $f_1^B$-value, contradicting our assumption $|V|\ge 3$. Consequently, we have $LO(x_{B_{k-i}}) \ge \ell-k$ for all $x \in P$. Since $|V| \ge 3$,  by the pigeon-hole principle, we now have two different values $u,v \in V$ and $x, y \in P$ such that $f_1^B(x_{B_{k-i}}) = u$, $f_1^B(y_{B_{k-i}}) = v$, and $(x_{B_{k-i}})_{\ell-r+1} = (y_{B_{k-i}})_{\ell-r+1}$. Since $f_1^B(x_{B_{k-i}}) = u \neq v = f_1^B(y_{B_{k-i}})$, the bit strings $x_{B_{k-i}}$ and $y_{B_{k-i}}$ differ in some position higher than $\ell-r+1$. We are thus in the third case of Lemma~\ref{lem:fbasic}. All three subcases of this case lead to a situation in which $(f_1^B(x_{B_{k-i}}),f_2^B(x_{B_{k-i}}))$ and $(f_1^B(y_{B_{k-i}}),f_2^B(y_{B_{k-i}}))$ are not non-dominating. By Lemma~\ref{lem:f}, this property extends to $x$ and $y$, contradicting our assumption on $P$. This terminates the proof.  
\end{proof}

\section{Runtime Analysis}
\label{sec4}

In this section, we prove our mathematical runtime results. We start with a lower bound for the classic GSEMO and continue with a (smaller) upper bound for the BC-GSEMO.

\subsection{Lower Bound for the GSEMO}

We prove the following lower bound on the runtime of the GSEMO.
\begin{theorem}\label{thm:lower}
  Let $n \in \N$ and $k = o(\sqrt{n / \ln(n)}\,)$ such that $k$ divides~$n$. Write $\ell = n/k$. Let $r = o(\ell)$ and $r \ge \max\{8+2k/\log_2(n), \gamma \ln(n)\}$ for some sufficiently large constant $\gamma$. Then the expected optimization time of \gsemo on $f_{nkr}$ is $\Omega(2^k n \ell)$.
\end{theorem}

As for many lower bounds proofs, the proof of this result is quite technical. Our main argument shall be that at some time, we have a population consisting of $\Omega(2^k)$ individuals such that least one of them misses at a constant fraction of ones in the cooperative segment of the last block. Finding these $\Omega(\ell)$ ones takes $\Omega(2^k n \ell)$ time, since a particular bit is flipped with probability $1/n$ and this this particular parent is chosen for mutation only with probability $O(2^{-k})$ (uniform choice among the $\Omega(2^k)$ individuals). The main work is showing that such an individual exists, which means that the later bits are not optimized to a large extend while the earlier bits are optimized. Also, we need to take care that not too many bits are optimized in one go by simply generating such an offspring from a different, already more optimized parent. 

In the following, we shall always assume that the assumptions of Theorem~\ref{thm:lower} hold. Being an asymptotic statement, we can and will always assume that $n$ is sufficiently large. We use the elementary notation from the combinatorics of words, e.g., that $1^i$ denotes the bit string $(1, \dots, 1) \in \{0,1\}^i$, or that for $x \in \{0,1\}^i$ and $y \in \{0,1\}^j$ the expression $xy$ denotes the bit string $(x_1, \dots, x_i, y_1, \dots, y_j) \in \{0,1\}^{i+j}$. We say that some individual $x \in \{0,1\}^n$ is of \emph{type} $t \in \{0,1\}^i, i \le n,$ to express that $x$ agrees with $t$ on its first $i$ elements, that is, $(x_1, \dots, x_i) = t$.

We say that a block $x_B$ of some individual $x$ is \emph{regular} if it has at most $d = \lceil 0.7 \ell \rceil$ leading ones and all bits with index higher than $d+1$ in this block are independently distributed with probability of being one at most $q = 0.85$.

We note that our following proofs use some of the main ideas of the proof of the result that the \oea optimizes a weighted sum of \leadingones blocks in a sequential fashion~\cite{DoerrSW13foga}. Unfortunately, due to the additional challenges imposed by the multi-objective setting, we cannot directly reuse the proofs from there. Differently from there, we need to argue that all later blocks are not optimized, as this ensures that only a single individual of a certain type is in the population; in~\cite{DoerrSW13foga}, it was only argued that the next block is not optimized. Also, having a non-trivial population size, we need to take care that a particular individual does not profit from being overwritten by an offspring from a different parent. For this reason, we have to argue that different individuals, should they exist, are sufficiently different so that it is unlikely that a mutation changes them enough to overwrite another individual. Such problems, naturally, do not appear in the $(1+1)$-setting regarded in~\cite{DoerrSW13foga}. For these reasons, while reusing some broader lines of the proof, we have essentially to come up with a new proof. On the positive side, our proof arguments could easily be used in the setting of~\cite{DoerrSW13foga}, and would there give a proof maybe more accessible than the one given in~\cite{DoerrSW13foga}.

\begin{lemma}\label{lem:technical}
  Consider a run of the GSEMO on our benchmark $f_{nkr}$ with parameters $n$, $k$, and $r$ as in Theorem~\ref{thm:lower}.
  Assume that for some iteration $t_0$ there is a $b \in [0..k-1]$ such that the following holds.
  \begin{itemize}
  \item An individual $x$ with type $1^{b\ell}$ has just entered the population.
  \item All blocks higher than $b$ are regular.
  \item If $b \ge 1$, then for all $i \le b$ an individual of type $1^{i\ell-r} 0^{r/2}$ is contained in the population.
  \end{itemize}
  Let $t_1$ be the first iteration in which an individual $y$ of type $1^{(b+1)\ell}$ is generated (and thus enters the population). Then with probability at least $1 - O(kn^{-2})$, the following properties hold.
  \begin{itemize}
  \item All blocks higher than $b+1$ are regular.
  \item For all $i \le b+1$ an individual of type $1^{i\ell-r} 0^{r/2}$ is contained in the population.
  \end{itemize}
\end{lemma}

\begin{proof}
  We first argue that with probability at least $1-O(n^{-2})$, no parent different from an individual of type $1^{b\ell}$ will generate an offspring that of this type in the next $n^2 2^k$ iterations. To this aim, we note that the property we initially assumed, that individuals of type $1^{i\ell-r} 0^{r/2}$ are present, remains true for the remaining run of the algorithm, simply because such individuals can only be replaced by others of this type. From this, we conclude that any individual in the population that is not of type $1^{b\ell}$ differs from any such individual in at least $r/2$ of the first $b\ell$ bits. 
  Consequently, the probability that such a parent generates a $1^{b\ell}$ offspring in one iteration, is at most $n^{-r/2} \le n^{-4} 2^k$ by our assumption on~$r$. Hence the probability that this happens in the next $n^2 2^k$ iterations, is at most $n^{-2}$. Since, as we will see shortly, the time interval analyzed in this lemma spans at most $O(\ell n)$ iterations with $x$ as parent (with probability $1- O(n^{-2})$), hence $O(\ell n 2^k)$ with any parent with this probability, we can assume in the remainder that never an individual of type $1^{b\ell}$ is generated from a parent of a different type.
 
  We start by analyzing the time interval until for the first time an individual with type $1^{(b+1)\ell-r}$ is in the population. We first analyze a fake version of our algorithm which never adds a second individual of type $1^{b\ell}$ to the population (hence when such an individual, incomparable with the existing one, is generated, it is just discarded instead of added to the population). We shall later see that this fake algorithm with high probability behaves identical to the true algorithm, so our insights extend to the true algorithm.

  In the fake algorithm, up to the point point where for the first time an individual of type $1^{(b+1)\ell-r}$ enters the population, $x$ is the only individual of type $1^{b\ell}$ in the population and  stays in the population unless replaced by its own offspring. Hence we can analyze the development of $x$ independent from all other individuals. To ease the language, we shall always denote the individual of type $1^{b\ell}$ by $x$. Let us also count time starting at $t=0$ and let us count only iterations in which $x$ is chosen as parent and in which no leading one of $x$ is flipped (as these will surely give an offspring that is not accepted). 

  We first note that each (such) iteration has a probability of exactly $\frac 1n$ of increasing the number of leading ones in~$x$. Hence the expected time to finish this phase is at most $n(\ell-r) \le n\ell$ and, by a simple Chernoff bound, this phase ends after at most $2n\ell$ iterations with probability at least $1 - \exp(\Theta(-\ell))$. Let us assume this in the remainder.

  Let us now regard a fixed block $x_B$ higher than $b$. We first note that with probability at least $(1 - \frac 1n)^n = (1+o(1)) \frac 1e$, the offspring of $x$ is accepted. Hence with probability at least $1 - 2n\ell \exp(-\Omega(n)) = 1 - \exp(-\Omega(n))$ in each interval of $n$ consecutive iterations in this phase, there are at least $\frac{n}{2e}$ iterations accepting the offspring. Let us also assume this.

  Let us partition this phase into consecutive intervals $I_1, I_2, \dots$ of exactly $2n$ iterations (plus some leftover of at most $2n-1$ iterations). Note that $L \le \ell$ by the assumptions made above. With probability at least 
  \begin{align*}
  (1-(1- &\tfrac 1n)^n)(1 - (1-\tfrac 1n)^{\ell/2}) (1 - \tfrac 1n)^n \\
  &\ge (1 - \tfrac 1e) \tfrac{\ell}{2en} (1-o(1)) \tfrac 1e \\
  &= (1-o(1)) \tfrac{1}{2e^2} (1 - \tfrac 1e) \tfrac{\ell}{n},
  \end{align*}
  in one of the first $n$ iterations of this interval, let us call it $t_u$,  
  \begin{enumerate}
  \item we flip the first zero of $x$ (implying that the offspring is accepted),
  \item we flip at least one of the first $\ell/2$ bits of $x_B$ should they all be one, or we leave one such bit at zero should such a bit exist; let us call the first zero bit in this block after this iteration $x_{u}$,
  \item and we do not flip $x_u$ in the next $n$ iterations (which still belong to this interval of $2n$ iterations).
  \end{enumerate} 
  Let us call this event a \emph{reset}, and take note that we have just shown that in each time interval $I_j$, $j = 1, \dots$, at most $\ell$ in number, independently (by definition of the reset), with probability at least $(1-o(1)) \frac{1}{2e^2} (1 - \frac 1e) \frac{\ell}{n}$ a reset happens. Consequently, when $C$ is a sufficiently large constant, with probability at least $1 - n^{-2}$ in each consecutive set of $C\frac n\ell \log n$ intervals, a reset happens. In terms of iterations, this means that with this probability, in each interval $[iL+1..(i+1)L]$ of iterations, $L := 2C\frac{n^2}{\ell} \ln(n)$, a reset happens. We study the effect of a single reset and then how these frequent resets influence $x_B$.
  
  In case of a reset, all bits higher than $x_u$ in this block are neutral in iterations $t_u+1, \dots, t_u+n$. Since at least $\frac n{2e}$ of these iterations accept the mutant by our earlier assumption, each of these neutral bits for at least $\frac{n}{2e}$ times is flipped independently with probability $\frac 1n$. By Lemma~3 of \cite{LassigS10}, this results in each of these bits, independently, being one with probability at most $\frac 12 + \frac 12 (1- \frac 2n)^{n/2e} \le \frac 12 (1 + \exp(-1/e)) \le 0.85 = q$. 

  Let us call a block $R$-regular if it contains a zero in the first $R$ bits and if all bits higher that $R+1$ are one with probability at most~$q$, independently. Hence in the notation above, a reset at time $t_u$ creates a $u$-regular block at time $t_u+n$, and this is also an $\frac \ell 2$-regular block. 
  
  We have shown above that (i)~after the start of the phase and  (ii)~after each reset we have at most $2L$ iterations without a reset (note that this argument is valid including the last few leftover iterations not fitting into a full interval $I_{j}$). 
  Consider for the moment any interval of $2L$ iterations (independent from the reset logic just set up). In each of these iterations, the probability that the number of leading ones in this block increases, is at most $\frac 1n$. Hence we expect at most $2\frac{L}{n} = 4C\frac{n}{\ell} \ln(n)$ such increases. Since bits above the first zero are one with probability at most $q$, the number of leading ones gained in such an improvement is $1+\Geom(q)$, where $\Geom(q)$ denotes a random variable, independent from all others, following a geometric law with success rate $q$. Consequently, the expected total gain of the number of leading ones in all these improvements is at most  $(1+\frac 1q) 4C\frac{n}{\ell} \ln(n)$ and a Chernoff bound for geometric random variables gives that with probability at least $1 - \exp(-\Omega(\frac{n}{\ell} \ln(n)))$, the total gains is bounded by  $(1+\frac 1q) 8C\frac{n}{\ell} \ln(n) = O(\frac n\ell \log n)$. By assumption, this gain is $o(\ell)$, which shows that up to the first reset the block is $r$-regular for some $r = (0.7+o(1))\ell$, from then on it is $r$-regular for some $r = (0.5+o(1)) \ell$. By taking a union bound over the at most $k$ later block, we extend this insight from one block to all later blocks, with probability at least $1 - O(kn^{-2})$.
  
  From what we have just shown, in particular, we observe that we never see that the number of leading ones in a later block reaches $\ell-r$, which would be necessary to generate a second individual of type $1^{b\ell}$ in the population. Hence with probability at least $1 - O(kn^{-2})$, also in the true algorithm we never see a second individual of this type, and thus our results above extend to the true algorithm. 

  We finally analyze the time until we have developed the individual of type $1^{(b+1)\ell-r}$ into one of type $1^{(b+1)\ell}$. We sketch this part as it mostly reuses arguments seen before. The difference is that now we will have two individuals of type $1^{(b+1)\ell-r}$, one collecting ones and the other collecting zeroes in the last $r$ bit positions. Both develop at a comparable speed, increasing the number of correct bits by an expected value of $\Theta(1/n)$ when counting only the iterations regarding these two individuals. Also generating the second from the first has a probability of $\Theta(1/n)$. Since $r \ge \gamma \ln n$ for some sufficiently large constant~$\gamma$, the probability that one individual is significantly faster than the other is $O(n^{-2})$. Since this phase is short, $O(nr)$ iterations with probability $1 - O(n^{-2})$, also the later blocks with high probability do not gain more than $o(\ell)$ leading ones.

  This proves the claim.
\end{proof}

\begin{proof}[Proof of Theorem~\ref{thm:lower}]
Consider a run of the GSEMO on our benchmark. Trivially, at the start of the run the conditions of Lemma~\ref{lem:technical} are satisfied for $b=0$. Also, if at some time $t_0$ the conditions of this lemma are satisfied for some $b \le k-1$, then the lemma asserts that at some later time $t_1$, they are satisfied for $b+1$, apart from a failure probability of $O(k/n^2)$. Hence by induction, using a union bound for the failure probabilities, we see that with  probability at least $1 - O(k^2/n^2)$, at the moment $t$ when an individual $y$ of type $1^{(k-1)\ell}$ enters the population, this individual has the last $0.3 \ell-1$ bits independently distributed with a probability of at most $q=0.85$ of being one. Also, at this time for each $i \in [1..k-1]$ the population contains an individual of type $1^{i\ell - r} 0^{r/2}$. As in the proof of Lemma~\ref{lem:technical}, this implies that the population contains no individual $z \neq y$ such that $z$ and $y$ differ in fewer than $r/2$ of their first $n-\ell$ bits. 

Let $S$ be the set of types of length $(k-1)\ell$ that can be extended to a Pareto optimal individual, that is, 
\begin{align*}
S = \{x \in &\{0,1\}^{(k-1)\ell} \mid \forall b \in [1..k-1] :  \\
&x_{(b-1)\ell+1} = \dots = x_{b\ell-r} = 1 
\wedge x_{b\ell-r+1} = \dots = x_{b\ell}\}.
\end{align*}
We note that the GSEMO treats the different types in $S$ in a symmetric fashion (see~\cite{Doerr21symmetry} for a formal treatment of symmetry arguments in the analysis of evolutionary algorithms) and that the types appear one after the other in the population (since only one offspring is generated per iteration). Consequently, the position of type $1^{(k-1)\ell}$ in the order of appearance of all types in $S$ is uniformly distributed in $[1..|S|]$. It thus appears, with probability at least $1/2$, on position $(|S|+1)/2$ or higher. Hence with probability at least $1/2$, the population size from this time on is at least $(|S|+1)/2 \ge 2^{k-2} =: m$.

Let us condition on this event. We analyze how the individual $y$ with type $1^{(k-1)\ell}$ develops into a solution of type $1^{n-r}$. 
We note first that the probability $Q$ that in one iteration a parent with type different from $1^{(k-1)\ell}$ generates an offspring of type $1^{(k-1)\ell}$ is at most $Q \le n^{-r/2}$. 
Hence the probability that this happens in the next $T_0 := 0.05\ell m n$ iterations is at most $Q T_0 \le n^{-r/2} 0.05 \ell 2^k n \le 2^k n^{2-r/2}$. 
Since $r \ge 8+2k/\log_2(n)$, this is at most $n^{-2}$. 

Ignoring such events for the moment, we see that the only way towards having a solution of type $1^{n-r}$ in the population is to select the single individual of type $1^{(k-1)\ell}$ as parent and mutate it. 
Apart from the waiting time for selecting the parent, this process up to reaching the type $1^{k\ell-r}$ is identical to the optimization of the \leadingones benchmark, starting with an individuals with $(k-1)\ell$ leading ones, then $0.7\ell+1$ bits which we do not specify further, and then $0.3\ell-1$ bits which are, independently, one with probability~$q$ only. 
We shall not regard for our lower bound the time it takes to optimize the first $0.7 \ell+1$ bits in the last block. 
From that point on, the bits are one independently with some probability that is at most~$q$. 
Hence the $T$ time to reach an individual of type $1^{k\ell-r}$ stochastically dominates the sum of $0.3\ell-1-r$ independent random variables, each being a geometric distribution with success rate $\frac 1{mn}$ (with probability $1-q$) or being constant zero (with the remaining probability of $q$), see~\cite{Doerr19tcs} for more details on this argument).
Hence $T$ has an expectation of at least $(0.3\ell-1-r) (1-q) p n = (0.045-o(1)) \ell p n$ and with high probability (use a Chernoff bound for sums of independent geometric random variables, e.g., Theorem~1.10.32 in~\cite{Doerr20bookchapter}), $T$ is at least $0.002 \ell p n$. 

In summary, we see that with probability at least $1/2 - o(1)$, and this includes all small failure probabilities encountered along the proof, we need at least $T_0 = 0.002 \ell p n = \Omega(2^k \ell n)$ iterations to generate a solution with type $1^{n-r}$, and hence also to generate the Pareto-optimal solution $1^n$, which is needed to witness the Pareto front. This implies an $\Omega(2^k \ell n)$ bound on the expected runtime. 
\end{proof}

\subsection{Upper Bound for the BC-GSEMO}%Frank

We now prove an upper bound on the runtime of the \bcgsemo algorithm which is lower than the lower bound shown for \gsemo in the previous section.
The fraction of mutation steps that \bcgsemo spends on a particular block $b$ is roughly $1/k$ if the parameter $\temax$ is small compared to the overall runtime. We assume that the maximal number of consecutive steps $\temax$ that a block is chosen for mutation is small enough such that each block is mutated frequently enough during the phases considered in the upcoming proofs.
For example, any $\temax \leq \sqrt{n}$ meets our requirement.
We analyze the algorithm using different phases of $T$ steps. When doing so, we  always assume without loss of generality that $T/(k \cdot \temax)$ is a positive integer in order to ease the presentation. This implies that mutation is applied to each block exactly $T/k$ times.

We first show that \bcgsemo quickly obtains for each block a sufficiently large number of leading ones. The key argument is that blocks are optimized in parallel and during this process and that \bcgsemo works with a population size of $1$.
\begin{lemma}
\label{lem:up}
Let $r< \sqrt{\ell \log \ell}$.
Then with probability $1-o(1)$, \bcgsemo has obtained after  $cn(\ell -\sqrt{\ell \log \ell})=O(n \ell)$ steps, $c$ an appropriate constant,  a population of size $1$ where for the individual $x \in P$, $LO_{z_1}(x_{B_i}) \in [\ell - 2\sqrt{\ell \log \ell}, \ell - \sqrt{\ell \log \ell}]$, $1 \leq i \leq k$, holds.
\end{lemma}
\begin{proof}
We consider the first phase consisting of $T=cn(\ell -\sqrt{\ell \log \ell})$ steps, $c$ an appropriate constant, which implies that $c \ell(\ell -\sqrt{\ell \log \ell})$ mutations are applied to each block $i$.
We show that each block consists of $\ell - \Theta(\sqrt{\ell \log \ell})$ leading one bits after these $T$ steps.
If a block $i$ is selected for mutation, then we have a classical leading ones problem on $\ell$ bits with mutation probability $1/\ell$.
%This is obtained in time $O(n\ell)$ as the population size remains $1$ during this phase.
We have $r< \sqrt{\ell \log \ell}$ which implies that as long as no block has at least $\ell+1-\sqrt{\ell \log \ell}$ leading ones, the population size is $1$. Each block develops independently of the others and the probability that there is a block $b$ for which the single solution $x \in P$ obtained after $T$ steps  does not have the property $LO_{z_1}(x_{B_i}) \in [\ell - 2\sqrt{\ell \log \ell}, \ell - \sqrt{\ell \log \ell}]$ after $c\ell(\ell -\sqrt{\ell \log \ell})$  mutations to block $b$ is $o(1)$ using Theorem 9 in~\cite{DoerrJWZ13}. 
%\frank{Benjamin: Would b great if you could check the previous statement as it's based on your previous work.}

Hence, after $T=cn(\ell -\sqrt{\ell \log \ell})$ steps the population $P$ contains a single solution $x$ for which we have $LO_{z_1}(x_{B_i}) \in [\ell - 2\sqrt{\ell \log \ell}, \ell - 2\sqrt{\ell \log \ell}]$ for all $i \in \{1, \ldots, k\}$ with probability $1-o(1)$.
\end{proof}

The previous lemma showed that the population does not grow for a significant phase of the optimization process and that the different blocks have obtained already a large number of leading ones. As \bcgsemo does not destroy the progress made in any of the blocks we can show the following upper bound by taking into account the remaining work that needs to be done on each block and the maximum population size of $2^k$.

\begin{theorem}
\label{thm:up}
Let $r< \sqrt{\ell \log \ell}$ and $k \leq n^{1/3}$. Then with probability $1-o(1)$, the optimization time of \bcgsemo on $f_{nkr}$ is $O(2^k n \sqrt{\ell \log \ell})$. 
\end{theorem}

\begin{proof}

We already know that \bcgsemo obtains with probability $1-o(1)$ in time $O(n\ell)$ a population of size $1$ where each block has a leading ones value at least $\ell - 2\sqrt{\ell \log \ell}$. 
We consider the process after this population has been produced for the first time and investigate the remaining process.

In the second phase, the population size is at most $2^k$ and optimizing the remaining bits to agree with $z_1$ (or $z_2$) for each solution takes time $O(2^k n k \sqrt{\ell \log \ell})$ with high probability. We show this by optimizing the blocks in lexicographic order. 
Assume all blocks $b< i <k$ have already been fully optimized. Block $i$ has at least $\ell - 2\sqrt{\ell \log \ell}$ leading ones in each solution present in the population.

Consider a solution $x$ with profile $p_{i-1}(x)=(x_{B_1}, \ldots, x_{B_i}) \in \{z_1, z_2\}^{i-1}$ for the first $i-1$ blocks. 
Let $s^1=\max_{y \in P}\{ LO_{z_1}(x_{B_{i}}) \mid p_{i-1}(y) = p_{i-1}(x)\}$ and $s^2=\max_{y \in P}\{ LO_{z_2}(x_{B_{i}}) \mid p_{i-1}(y) = p_{i-1}(x)\}$. The values of $s^1$ and $s^2$ do not decrease and the probability of an increase for each of them is at least $1/(2^k e\ell)$ if block $i$ is selected for mutation.

For each value of $i$, $1 \leq i \leq k$, we consider a (sub)-phase of $T_i=cn 2^k \sqrt{\ell \log \ell}$ steps, $c$ an appropriate large enough constant, which starts when the population $P$ contains an individual for each profile $\{z_1, z_2\}^{i-1}$ for the first time.
A phase of $T_i$ steps implies that there are $c \ell 2^k \sqrt{\ell \log \ell}$ mutation steps applied to block $i$.
The expected number of mutation steps applied to each solution $x$ with profile $p_{i-1}(x) \in \{z_1, z_2\}^{i-1}$ and maximal $s^1$-value (or $s^2$-value) is at least $c \ell \sqrt{\ell \log \ell}$ and the expected number of increases of $s^1$ (or $s^2$) is at least $c'\sqrt{\ell \log \ell}$, $c'\geq 4$ a constant if $c$ is large enough. The probability that there are less than the sufficient  
$$2\sqrt{\ell \log \ell} \leq \frac{c'}{2} \cdot \sqrt{\ell \log \ell}$$ 
successful increases of $s^1$ (or $s^2$) is $2^{-\hat{c}\sqrt{\ell \log \ell}}$ using Chernoff bounds where $\hat{c}>1$ if $c$ is chosen large enough. 
Assuming $k\leq n^{1/3}$ which implies $\ell \geq n^{2/3}$ and using the union bound, the probability that there is a profile $p_{i-1}(x)$ for which $s^1=\ell$ (or $s^2 = \ell$) 
has not been achieved is at most 
$$2^{n^{1/3}} \cdot 2^{-\hat{c} \sqrt{n^{2/3}}}= 2^{n^{1/3}} \cdot 2^{-\hat{c} n^{1/3}} =2^{(1-\hat{c})n^{1/3}} = 2^{-\Omega(n^{1/3})}
%2^{n^{1/3}} \cdot 2^{-\hat{c}n^{1/3} \sqrt{\log \ell}} = %e^{-\Omega(n^{1/3} \sqrt{\log \ell})}.
$$ 
%for $\hat{c}>1$ when the constraint $c$ is chosen large enough.
%O(1/n^{\hat{c}})$, where $\hat{c}$ is a constant that increases with %$c$. 
%\frank{Benjamin: could you check this?}
Considering all values of $i$, $1\leq i \leq k$, sequentially, the probability that there is a block $i$ which has not been fully optimized its phase of $T_i$ steps is $k\cdot 2^{-\Omega(n^{1/3})} = 2^{-\Omega(n^{1/3})}$. As there are $k$ blocks that need to be considered, an optimal population is reached after at most $\sum_{i=1}^k T_i= cnk 2^k \sqrt{\ell \log \ell}$ steps with probability $1-2^{-\Omega(n^{1/3})}$.
Taking into account the failure probability of $o(1)$ for the first phase given in Lemma~\ref{lem:up} completes the proof.
\end{proof}

\section{Experimental investigations}
\label{sec5}
We now complement our theoretical findings with an experimental study. In order to show that a significant difference can already been observed for a moderate number of blocks, we consider $k=2,3,4$ for values of $n= 24, 120,240, 360, 480, 600, 720, 840$. We consider problems with a value of $r=1,2,4$ and use $\temax=1000$ for \bcgsemo. We obtained similar results as shown in the following when using $\temax=1, 100$ in \bcgsemo. 

The results are shown as the mean number of function evaluations  required to reach the optimal population for $30$ independent runs. 
As there is a clear difference in the scaling behaviour of \gsemo and \bcgsemo for all experimental investigations carried out, we do not report statistical test results. (Note: the standard error on the mean is too small to be visible, so we omit to plot it.) 
\begin{figure}[!t]
\centering
\includegraphics[width=0.40\textwidth]{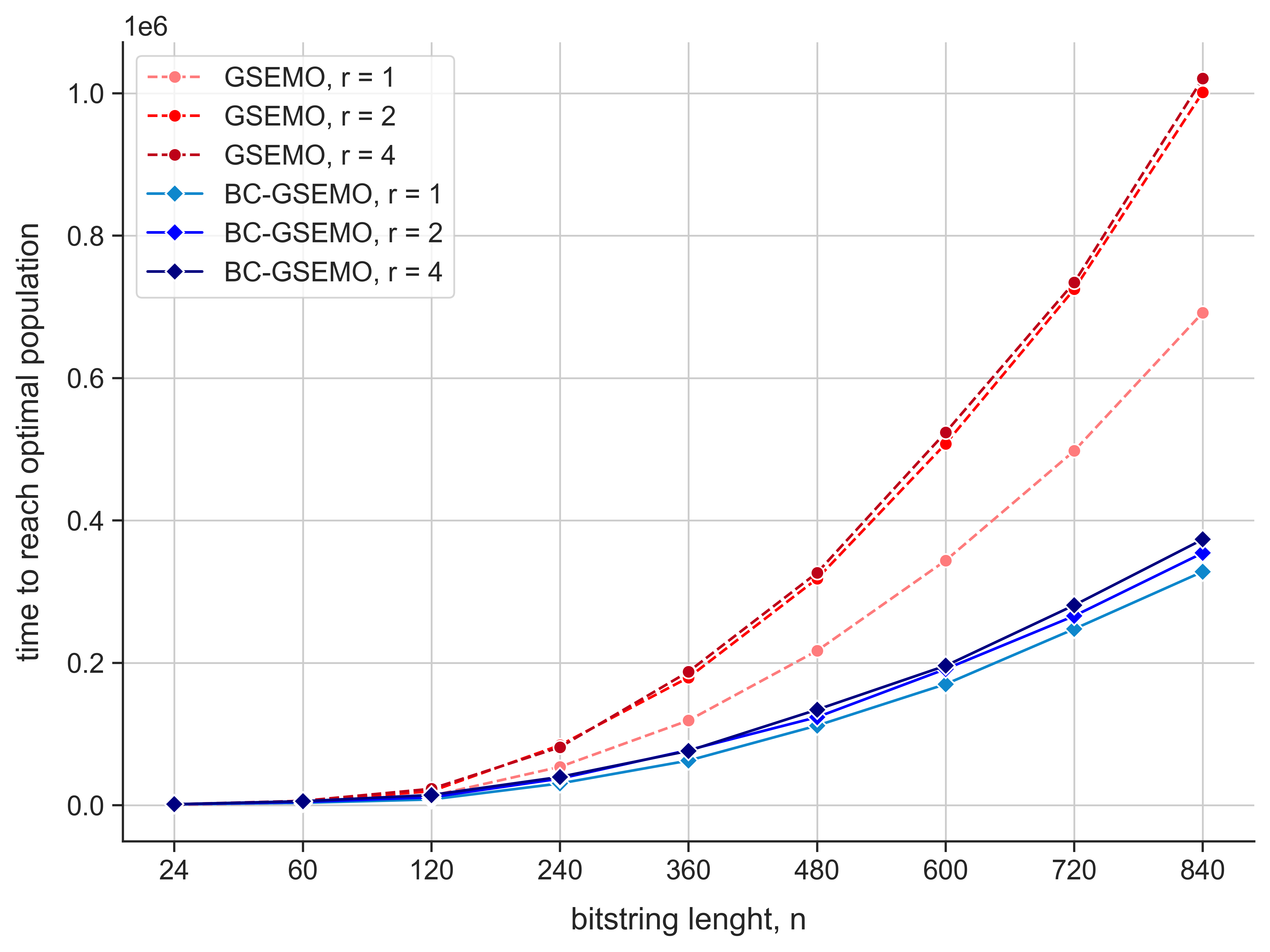}
\caption{Mean time to reach the optimal population by GSEMO and BC-GSEMO (with $\temax =1000$) for $k = 2$ dependent on the bitstring length $n$ for $r = 1, 2, 4$.}
\label{fig:k2t1000}
\end{figure}

\begin{figure}[!t]
\centering
\includegraphics[width=0.40\textwidth]{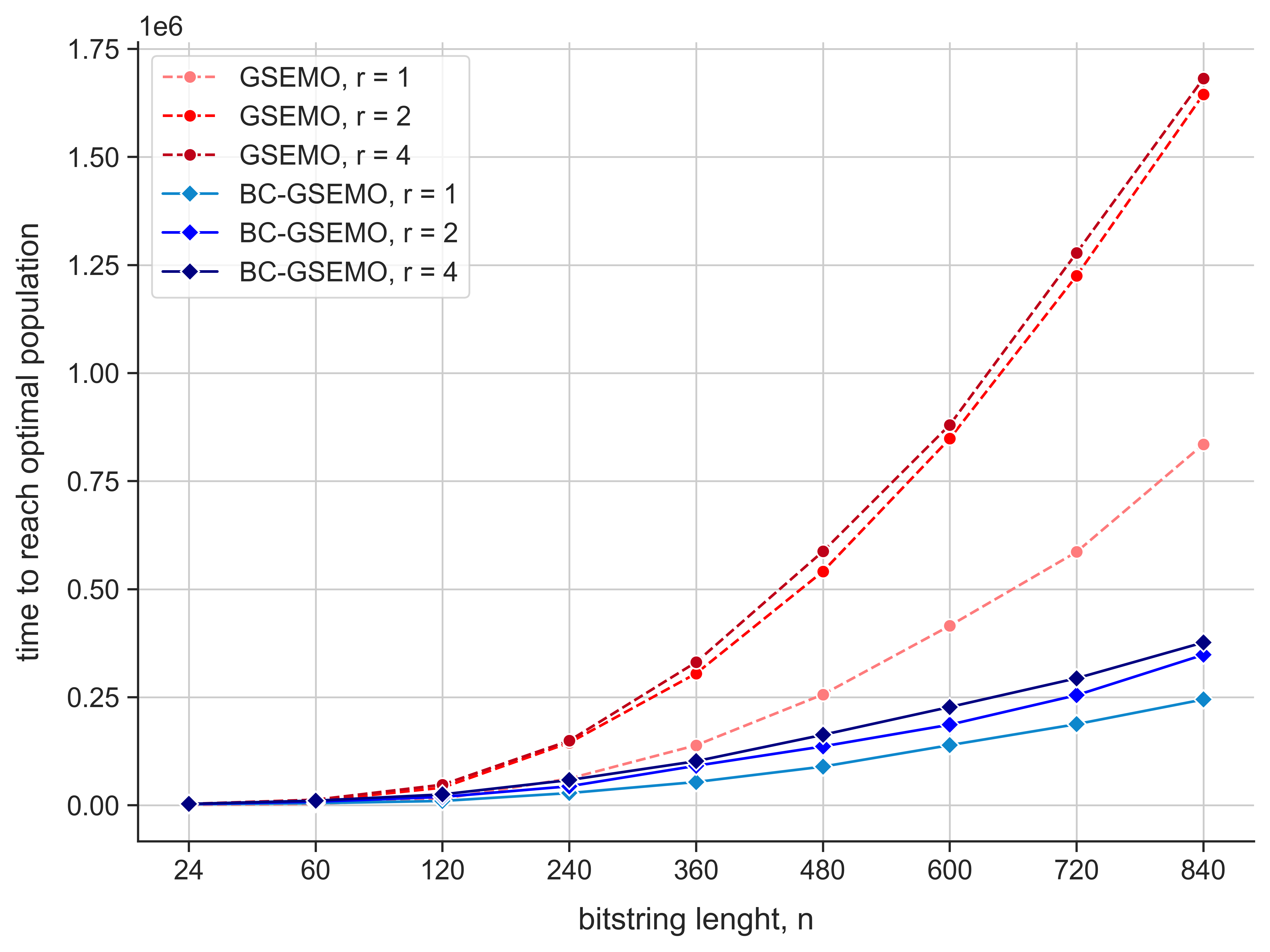}
\caption{Mean time to reach the optimal population by GSEMO and BC-GSEMO (with $\temax =1000$) for $k = 3$ dependent on the bitstring length $n$ for $r = 1, 2, 4$.}
\label{fig:k3t1000}
\end{figure}

\begin{figure}[!t]
\centering
\includegraphics[width=0.40\textwidth]{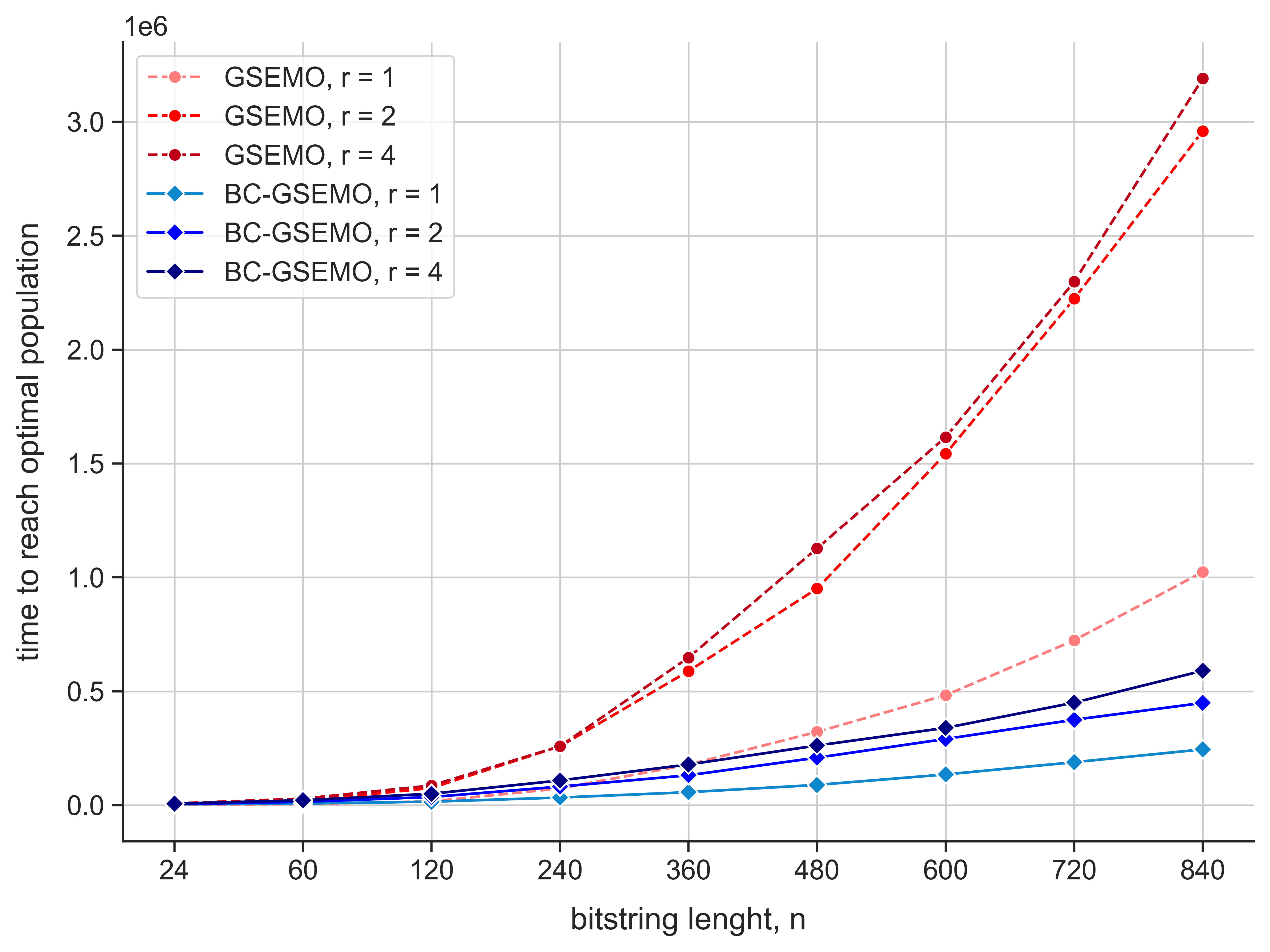}
\caption{Mean time to reach the optimal population by GSEMO and BC-GSEMO (with $\temax =1000$) for $k = 4$ dependent on the bitstring length $n$ for $r = 1, 2, 4$.}
\label{fig:k4t1000}
\end{figure}

The results for $k=2$ and $\temax=1000$ are shown in Figure~\ref{fig:k2t1000}. It can be observed that \bcgsemo needs much fewer function values than \gsemo to compute the optimal population. Looking at the results of \gsemo for different values of $r$, we can see that a larger value of $r$ results in a larger runtime.
Increasing $r$ from $1$ to $2$ for \gsemo, this increases the runtime significantly whereas the increase is rather minor when further increasing $r$ from $2$ to $4$.
This observed behaviour is expected from the problem definition and is also in line with our investigations for the lower bound of \gsemo, where we required that $r$ grows with $k$. The increase in runtime with increasing value of $r$ can also be observed for \bcgsemo.

Furthermore, the results for $k=3,4$ blocks and $\temax=1000$ are shown in Figure~\ref{fig:k3t1000} and Figure~\ref{fig:k4t1000}, respectively. We can observe a similar behavior in terms of runtime for both algorithms. \bcgsemo is consistently superior to \gsemo in terms of runtime, demonstrating its efficiency across various
settings. Again the runtime for \gsemo significantly increases when increasing the problem parameter $r$ from $1$ to $2$ whereas this difference is not that larger when increasing $r$ from $2$ to $4$. For \bcgsemo we can observe an increase in runtime when increasing $r$ which is not as drastic as for \gsemo. This increase is again larger when increasing $r$ from $1$ to $2$ than when increasing it further from $2$ to $4$.
Comparing the results for different values of $k$ across all three figures, we see that the runtimes are clearly increasing with $k$ across all other problem parameter settings with the increase being significantly larger for \gsemo than for \bcgsemo.

\section{Discussion and Conclusions}
Through an example problem, we have shown that incorporating block coordinate descent into evolutionary multi-objective algorithms is provably useful. Our result is based on the insight that block-coordinate descent allows for the parallel optimization of the underlying blocks without destroying other components. In practical problems, we may need methods for identifying the blocks automatically if the problem has no obvious natural decomposition. Estimation of distribution algorithms (EDAs) that use factorised distributions are candidates for identifying blocks. A particularly promising candidate is Deep Optimization~\cite{caldwell2022deep}, an EDA that identifies and constructs latent variables hierarchically, resulting in lower and lower dimensional spaces to search in. Such latent variables could be interpreted as blocks. Empirically, Deep Optimization exhibits a favourable scaling behaviour on multiple knapsack and other problems  due to its ability to search by combining modules without repeatedly destroying them. 

Of course, there remains a gap between the simple function and algorithms we have fully analysed here, and more sophisticated, practical algorithms like Deep Optimization which presently remain out of reach of theoretical running time analysis. Nevertheless, our result, as well as providing an existence proof for the efficiency of block-coordinate methods in a multi-objective setting, also suggests a lens by which to understand such algorithms, and points the way to uncovering more about what may make these kind of algorithms efficient in practice.

\section*{Acknowledgments}
The research has been initiated during Dagstuhl seminar 23361, ``Multiobjective Optimization on a Budget''.
Special thanks are due to the Working Group ``Decomposition'', which in addition to the authors included Kerstin D\"achert, Andrea Raith,  Anita Sch\"obel, and Margaret M.~Wiecek. Benjamin Doerr's research benefited from the support of the FMJH Program Gaspard Monge for optimization and operations research and their interactions with data science.
Frank Neumann has been supported by the Australian Research Council through grant FT200100536.

\end{document}